%% file: main.tex
\documentclass[conference]{IEEEtran}
\IEEEoverridecommandlockouts
\pdfoutput=1
\usepackage{cite}
\usepackage{amsmath,amssymb,amsthm,amsfonts,physics,booktabs,multirow,multicol,algorithmic, algorithm, url}
\usepackage{algorithmic}
\usepackage{graphicx}
\usepackage{textcomp}
\usepackage{xcolor}
\newcommand{\eat}[1]{}

\newtheorem{theo}{Theorem}
\newtheorem{lemm}{Lemma}

\usepackage{color,soul,xspace}
 
\def\BibTeX{{\rm B\kern-.05em{\sc i\kern-.025em b}\kern-.08em
    T\kern-.1667em\lower.7ex\hbox{E}\kern-.125emX}}

\title{
Improving Your Graph Neural Networks: \\
A High-Frequency Booster 
}

\begin{document}



\newcommand*{\affaddr}[1]{#1}
\newcommand*{\affmark}[1][*]{\textsuperscript{#1}}
\newcommand*{\email}[1]{\texttt{#1}}

\author{
Jiaqi Sun\affmark[1], 
Lin Zhang\affmark[2], 
Shenglin Zhao\affmark[2], 
Yujiu Yang\affmark[1$^\dagger$]\thanks{$^\dagger$Corresponding author}\\
\affaddr{\affmark[1]Shenzhen International Graduate School, Tsinghua University, Shenzhen, China}\\
\affaddr{\affmark[2]Tencent, Shenzhen, China}\\
\email{sunjq20@mails.tsinghua.edu.cn,linzhang0529@gmail.com}\\
\email{zsl.zju@gmail.com,yang.yujiu@sz.tsinghua.edu.cn}
}

\maketitle

\input{tex/0_abstract}


\input{tex/1_introduction}

\input{tex/2_relatedwork}
\input{tex/3_preliminaries}

\input{tex/4_5_proposal}
\input{tex/6_experiment}
\input{tex/7_conclusion}

{\bibliographystyle{abbrv}
\bibliography{reference}}

\end{document}

%% file: tex/0_abstract.tex
\begin{abstract}

Graph neural networks (GNNs) hold the promise of learning efficient representations of graph-structured data, and one of its most important applications is semi-supervised node classification.
However, in this application, GNN frameworks tend to fail due to the following issues: over-smoothing and heterophily.
The most popular GNNs are known to be focused on the message-passing framework, and recent research shows that these GNNs are often bounded by low-pass filters from a signal processing perspective.
We thus incorporate high-frequency information into GNNs to alleviate this genetic problem.
In this paper, we argue that the complement of the original graph incorporates a high-pass filter and propose Complement Laplacian Regularization (CLAR) for an efficient enhancement of high-frequency components.
The experimental results demonstrate that CLAR helps GNNs tackle over-smoothing, improving the expressiveness of heterophilic graphs, 
which adds up to $3.6\%$ improvement over popular baselines and ensures topological robustness.
\end{abstract}

%% file: tex/1_introduction.tex
\section{Introduction}
Recent years have witnessed the prosperity of graph neural networks (GNNs) in various domains summarized in~\cite{2021GNNs}.
At an early age, GNN models often directly apply the full-scale spectrum filters from Laplacian eigenvectors, such as Cheby-GCN~\cite{DBLP:conf/nips/DefferrardBV16}.
Although theoretically sound, these methods have prohibitive computational costs due to the eigendecomposition.
To alleviate this issue, researchers propose to approximate filters and lead to modern GCN~\cite{DBLP:conf/iclr/KipfW17},
aggregating neighboring nodes' messages for updating the central nodes, called message-passing (MP) mechanism.
This mechanism is often framed as a regularization to integrate the graph information~\cite{DBLP:conf/iclr/KipfW17}.


Despite its success, such a regularization inherits issues from the message passing mechanism that only focuses on the low-frequency information {(i.e., the low-valued eigenvalues of the graph Laplacian matrix; related definitions are provided in the next section.)} in graphs~\cite{DBLP:conf/aaai/LiHW18}.
This often leads to severely degraded performance when stacking multiple layers, i.e., over-smoothing phenomena~\cite{DBLP:conf/icml/ChenWHDL20,DBLP:conf/icml/WuSZFYW19}.
Another vital bottleneck in applying existing GNNs is that graphs with heterophily are orthogonal to the underlying assumption of MP, where directly connected nodes in such graphs are not necessarily similar~\cite{DBLP:conf/aaai/BoWSS21}.
In the context of regularization, 
a line of research attempt to address these limitations, such as promoting consistency between layers~\cite{DBLP:conf/aaai/0002MC21}, randomly dropping 
edges~\cite{DBLP:conf/iclr/RongHXH20}.
However, these solutions require low-frequency filtering as an input from the beginning; therefore, the corresponding expressive ability is constrained due to the above issues.

\begin{figure}[t]
    \centering
    \includegraphics[trim=2.5cm 0cm 1cm 0cm,clip=true,width=0.48\textwidth]{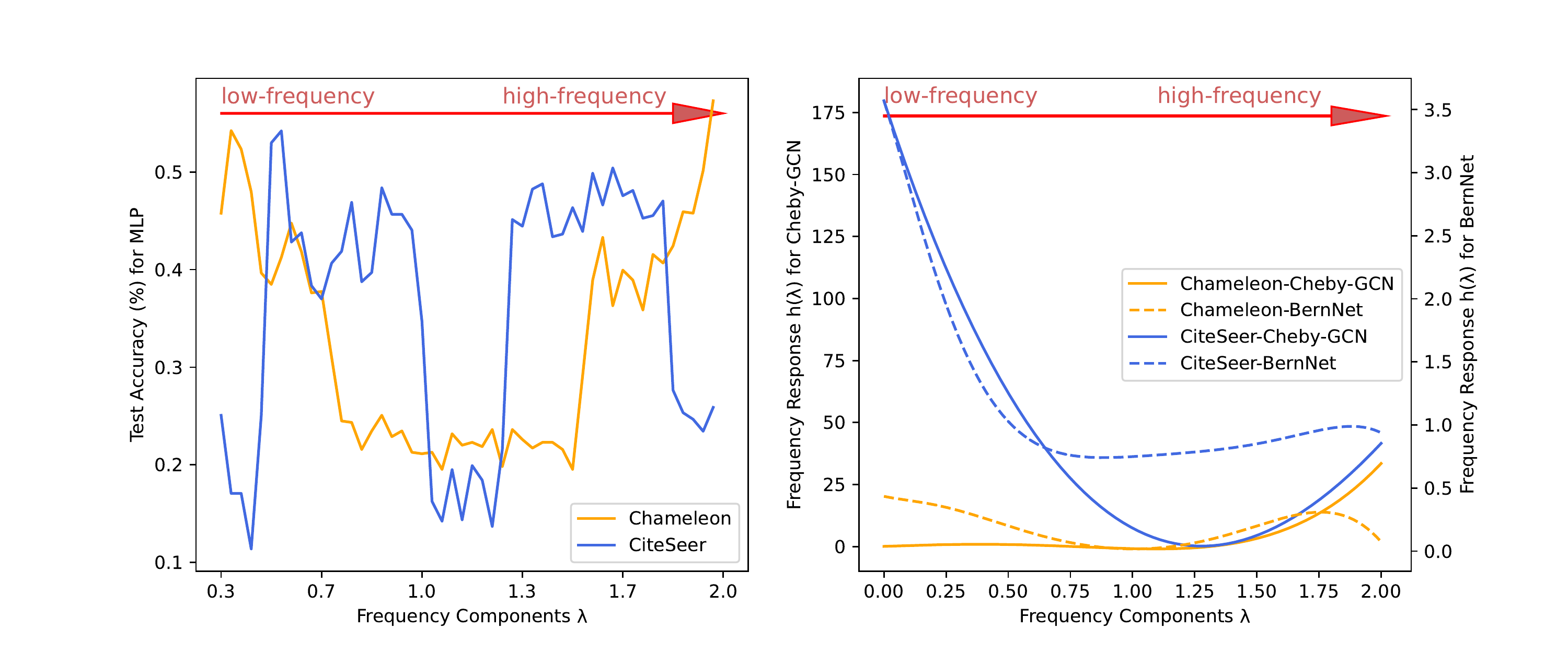}
    \caption{We illustrate the importance of high-frequency information in graphs by experimenting with Chameleon (heterophily) and CiteSeer (homophily) datasets. On the left, we show the node classification accuracy under different levels of frequencies.
    The performance from the high-frequency part alone is on par with that of low-frequency, suggesting discriminative power from high frequencies.
    On the right, GNNs with wide-band filters receive unneglectable responses
    from high frequencies, which is often abandoned in existing GNNs.
    }
    \label{fig:freq}
\end{figure}

\textit{Could we instead expand the range of frequencies in this setting?} 
Research shows that graph data often includes a wide range of frequencies~\cite{ortega2018graph}, and high-frequency components have advantages in certain tasks~\cite{fagcn2021}.
In particular, the aforementioned issues may not be as much of a concern under high frequencies.
Consider, as an example,  both
high-frequency components 
and low-frequency components from two different types of graphs: homophily and heterophily, as shown in Fig.~\ref{fig:freq}.
The classification performance varies from low-frequency to high-frequency zones, suggesting the importance of high frequency in this task.
Further, we observe comparable frequency response from the high-frequency level when using wide-band frequency-sensitive filters, such as GPRGNN~\cite{DBLP:conf/iclr/ChienP0M21} 
and Cheby-GCN~\cite{DBLP:conf/nips/DefferrardBV16}. 
This further strengthens the role of high-frequency components, which should be considered in the message passing procedure.

In recent years, there are also some efforts try to explore high-frequency in GNNs, e.g.,~\cite{DBLP:conf/aaai/BoWSS21,DBLP:conf/cikm/DongDJJL21}. These approaches are comprehensively concluded in the next section. 
\textit{However, they all implement new GNN architectures, and fail to adapt to existing models.} This is unwanted in many maturely practiced scenarios, where the necessary properties from existing models can not be abandoned.
%

Our goal in this work is to efficiently exploit the high frequency component to boost the express ability of GNNs while maintaining the merit of traditional GNNs. 
We develop Complement LAplacian Regularization (CLAR), which introduces the high frequency as a plugin adding to existing GNNs.
More specifically, in CLAR, we adopt 1) random sampling strategies to better capture the high-frequency components from the complement and 2) the original Laplacian regularization to balance the noisy connections from the sampling.
%
Our contributions are as follows:
\begin{itemize}
    \item  \noindent{\bf  Problem Formulation:} We 
    show the necessity of applying 
    high-frequency components and define the problem of integrating high-frequency in graph learning.
    \item  \noindent{\bf  Effective Algorithm:} We develop CLAR, a model-agnostic plugin that enhances the high-frequency components for GNNs with theoretical justification. 
    \item \noindent{\bf  Effectiveness:} Extensive experiments show that our solution enhances the performance of GNNs against over-smoothing, enhance the expressivity in the heterophilic graphs, and boost the robustness against topological noise.
    It also outperforms other regularization methods using the high-frequency information. 
\end{itemize}

%% file: tex/2_relatedwork.tex
\section{Related Work}
\label{sec:relate}

Under the message-passing mechanism, GNNs aggregate the neighbors' information and update the central nodes (formulated in Eq.~(\ref{equ:gnn_layer})), such as GCN~\cite{DBLP:conf/iclr/KipfW17}, GAT~\cite{DBLP:conf/iclr/VelickovicCCRLB18} using the attention mechanism, and SAGE~\cite{DBLP:conf/nips/HamiltonYL17} sampling the neighboring nodes. 

\textbf{Spectral explanations for GNNs.}\quad
GNNs apply (graph) filters on the full-scale eigenvalues/frequency components obtained from the graph Laplacian eigendecomposition~\cite{DBLP:conf/icml/WuSZFYW19,DBLP:conf/iclr/BalcilarRHGAH21}.
{The corresponding frequency components are of low-frequency and the high-frequency parts, as shown in Fig.~\ref{fig:freq}. Accordingly, we have filters of low-pass and high-pass when defining them on these frequency components.
For instance, in the right canvas of Fig.~\ref{fig:freq}, the denoted $\mathtt{Chameleon-Cheby-GCN}$ indicates a high-pass filter, since it specially lets high-frequency components to pass.}

For a deeper view of the property that GNNs as filters hold, \cite{DBLP:conf/aaai/LiHW18} prove that most GNNs focus on the low-frequency components and act as low-pass filters, they quickly become more low-pass after stacking layers, leading to over-smoothing. Furthermore, the low-frequency components encourage the connected nodes to become similar, which is orthogonal to heterophilic graphs where the connected nodes tend to be disparate~\cite{DBLP:conf/cikm/DongDJJL21,DBLP:conf/aaai/BoWSS21}.

{
\textbf{Architecture-based solutions.}\quad
%
Self-loops~\cite{DBLP:conf/iclr/XuHLJ19} and skip-connection~\cite{DBLP:conf/icml/XuLTSKJ18, DBLP:conf/icml/ChenWHDL20, DBLP:conf/iccv/Li0TG19} help to retain the initial node features for alleviating over-smoothing in GNNs. 
Dealing with heterophily, CPGNN~\cite{DBLP:conf/aaai/ZhuR0MLAK21} sieves the compatible neighbors in every propagation via learning a compatible matrix for each label. H2GCN~\cite{DBLP:conf/nips/ZhuYZHAK20} concatenates all layers to restore representations from the former aggregations. 

Another line of solutions analyze the causes of this problem,  i.e., the nature of low-pass filters in GNNs, and address it directly.
Such new architectures often incorporate high-pass filters to alleviate these issues.
FAGCN~\cite{DBLP:conf/aaai/BoWSS21} arranges low-pass and high-pass aggregations in each layer. AdaGNN~\cite{DBLP:conf/cikm/DongDJJL21} applies several filters in each aggregation and \cite{DBLP:conf/pkdd/LiKW21} further adapt the attention mechanism. 
\cite{chang2021spectral} leverages attention mechanism to achieve an high-pass filter feature.
Besides, the polynomial based graph filters, e.g.,  Cheby-GCN~\cite{DBLP:conf/nips/DefferrardBV16}, BernNet~\cite{DBLP:journals/corr/abs-2106-10994}, and GPRGNN~\cite{DBLP:conf/iclr/ChienP0M21}, have the potential to approximate arbitrary filters by utilizing high-ordered polynomials.

\textbf{Regularization-based solutions.}\quad
Some recent work develops regularizers for existing GNN models, making this strategy practical as it avoids designing from scratch.
 P-reg~\cite{DBLP:conf/aaai/0002MC21} promotes the consistency of adjacent layers to accelerate GNNs to the infinite layer. MADReg penalizes the smoothness of hidden representations, and AdaGraph enhances the graph topology via self-training~\cite{DBLP:conf/aaai/ChenLLLZS20}. DropEdge~\cite{DBLP:conf/iclr/RongHXH20} randomly drops a specific ratio of edges, a particular case of AdaGraph.
These regularizers align with the semi-supervised paradigm (Eq.~(\ref{equ:semi_loss})), nevertheless, they inherit the low-frequency components from GNNs and thus heterophily is unsolved.
}

\textbf{Our proposal to related work.}\quad
Our proposal, CLAR, is the first regularization method that derives from spectral domain to enhance high-frequency components for GNNs. 
Unlike architecture-based approaches, we incorporate high-pass filters in a plug-in manner, avoiding the work of implementing a new architecture, thus maintaining the properties of the backbone GNN.
Compared with existing regularization methods, we practice the high-pass feature advanced.

%% file: tex/3_preliminaries.tex
\section{Preliminaries}
\label{sec:pre}
\subsection{Graph Concepts}
\textbf{Notations.}\quad
$\mathcal{G}=\left(\mathcal{V},\mathcal{E}\right)$ is a connected graph with node features $\mathit{X} = \{x_1,\cdots,x_{\mathnormal{N}}\}$ and labels $Y = \{y_1, \cdots, y_{N}\}$ with $c$ categories.
$\mathit{A} \in \mathbb{R}^{\mathnormal{N}\times\mathnormal{N}}$ is the adjacency matrix, whose degree matrix is diagonal that $\mathit{D}_{ii}=\sum_{j=1}^{\mathnormal{N}}\mathit{A}_{ij}$.
The Laplacian matrix is $\mathit{L}=\mathit{D}-\mathit{A}$. 
$\mathrm{N}(\cdot)$ operates to get the neighbors, and 
$\mathrm{L}(\cdot)$ is to obtain a Laplacian matrix.
$\hat{\cdot}$ indicates add-self-loop operator, and $\tilde{\cdot}$ means normalization.
A normalized Laplacian matrix can be factorized as $\tilde{L}=U \Lambda U^T$, where $\Lambda$ is diagonal with entries $\Lambda_{ii} = \lambda_i$ indicating the frequency components/spectrum of $\mathcal{G}$.
%

\textbf{The complement graph.}\quad
The complement of $\mathcal{G}$ is denoted by $\bar{\mathcal{G}}=\left(\mathcal{V},\bar{\mathcal{E}}\right)$, where $\bar{\mathcal{E}}=\mathcal{V}\times\mathcal{V}-\mathcal{E}$, i.e., for any distinctive nodes $v_i, v_j \in \mathcal{V}$, $\left(v_i, v_j\right) \in \bar{\mathcal{E}}$ if and only if $\left(v_i, v_j\right) \notin \mathcal{E}$~\cite{1977Graph}.


\textbf{Homophily and heterophily.}\quad 
Homophily means that the connected nodes share the same label, while the labels of neighboring nodes may distinct from each other on heterophilic graphs. \cite{DBLP:conf/nips/ZhuYZHAK20} propose homophily ratio $h$ to measure homophily,
\begin{equation}
    h=\frac{|\{(v_i,v_j)|(v_i,v_j)\in\mathcal{E}\land y_i =y_j\}|}{|\mathcal{E}|}.
    \label{equ:homo}
\end{equation}
$h$ represents the ratio of edges with the same label over the total edges, and thus $h\in[0,1]$. Usually, we consider the graphs with $h>0.7$ are homophilic, and the ones with $h<0.3$ are heterophilic~\cite{DBLP:journals/corr/abs-2106-10994,DBLP:conf/aaai/BoWSS21}.

\subsection{Graph Neural Networks Layer}
GNNs utilize both the node features and the graph structure for learning representations. In each layer, GNN aggregates the neighboring nodes and update the central nodes. 
Given the representation of node $v_i$ at $k$-th layer $h_i^{(k)}$, and the neighboring nodes of $v_i$, ${N}(v_i)$, GNN formulates the hidden representation of $v_i$ at layer $k+1$ as 
\begin{equation}
    h_i^{(k+1)}=\sigma\left(\mathtt{UPD}\left(\mathtt{AGG}\left(\{h_j^{(k)}\}\right), h_i^{(k)} \right)W^{(k+1)}\right),
    \label{equ:gnn_layer}
\end{equation}
where $\sigma$ is the non-linear activation, $\mathtt{UPD}$ and $\mathtt{AGG}$ are the aggregation and updating functions, $\{h_j^{(k)}\}$ is the hidden representations for the neighbors of node $v_i$, that $\{v_j|v_j\in \mathrm{N}(v_i)\}$. $W^{(k+1)}$ is the feature transformation parameter. 
The input is initialized as the raw node feature as $h_i^{(0)}=x_i$.
Functions like $\mathtt{UPD}$ and $\mathtt{AGG}$ process node information together with the adjacency matrix $A$, 
resulting in a recursive updating manner as $H^{(k+1)}=f^{(k+1)}(H^{(k)}, A)$,
where  $f^{(k+1)}(\cdot)$  denotes parameter learning at the $k+1$-layer.
Then, the overall GNNs stack $K$ layers and return $ H^{(K)}$,
\begin{equation}
    H^{(K)} = f^{(K-1)}( \cdots f^{(1)}(f^{(0)}(X,A),A) \cdots ,A).
    \label{equ:gnn_fx}
\end{equation}

\subsection{Semi-supervised Node classification}
Recall that a standard semi-supervised node classification is often formulated as 
\begin{equation}
    \mathcal{L}=\mathcal{L}_{cls}+\gamma\mathcal{L}_{reg}, 
    \label{equ:semi}
\end{equation}
where $\mathcal{L}_{cls}$ optimizes the objective for classification, e.g., a multi-layer perceptron (MLP); $\mathcal{L}_{reg}$ promotes the similarity of the connected nodes' representations, writing as
\begin{equation}
    \mathcal{L}_{reg}=\tr\left(g\left(X\right)^TLg\left(X\right)\right),
    \label{equ:reg}
\end{equation}
where $\tr(\cdot)$ calculates the trace of a matrix and $g(\cdot)$ can be any transformations on the initial features. 

%% file: tex/4_5_proposal.tex
\section{Our Proposal: Complementary LAplacian Regularization}
\label{sec:prop}
\newtheorem{lemma}{Lemma}
\newtheorem{remark}{Remark}
\newtheorem{claim}{Claim}
\newtheorem{theorem}{Theorem}
We propose CLAR to replenish the high-frequency components for GNNs via a complementary Laplacian regularization. 
In the following, we first introduce our main theory in Theorem~\ref{theorem}, showing 
complementary Laplacian regularization incorporates a high-pass filter.
Next, we present the implementation details of CLAR, including sampling strategies and constructing the regularization. 

\subsection{Main theorem} 

Suppose $W^*$ is the optimal feature transformation parameters and initialize $Z^{(0)}=XW^*$ for aggregation. \cite{DBLP:conf/www/ZhuWSJ021} propose:
\begin{lemm}
\label{lemma2}
Given an adjacency matrix $A$,
a GNN model can be used to approximate a regularization objective that is defined on the corresponding Laplacian matrix $\mathnormal{L}(A)$.
\end{lemm}

Besides,~\cite{DBLP:conf/iclr/BalcilarRHGAH21} link the aggregation functions to the frequency response of spectral filters.

\begin{lemm}
\label{lemma1}
Given an adjacency matrix $A$, 
the frequency response $\Phi_{A}$  of 
a GNN defined on $A$ is 
negatively related to the frequency components of Laplacian matrix of $A$, i.e., $\Phi_{A}(\lambda) \approx 1-\lambda$ is a low-pass filter.
\end{lemm}

So far, we have bridged the GNN regularization objective to the following: i) the GNN aggregation function, and ii) the frequency response. 
We take GCN for an example to instantiate this relationship.
Using Lemma~\ref{lemma2}, the regularization objective for GCN is 
\begin{equation}
    \mathcal{L}_{GCN}=\min_{Z} \tr(Z^T\tilde{\mathrm{L}}(A)Z).
\label{equ:gcn_optim}
\end{equation}
{
With this formulation, we can view the mean-average aggregation of GCN (using $\tilde{A}$) as a regularization objective in terms of the normalized Laplacian matrix $\tilde{\mathrm{L}}(A)$.
}
Next, suppose the frequency components of $\tilde{\mathnormal{L}}(A)$ is $\lambda$, where we remove the subscript to denote the variable of frequency components. From Lemma~\ref{lemma1}, GCN acts as a low-pass filter that $\Phi_{A}(\lambda) \approx 1-\lambda$, aligning with the results from \cite{DBLP:conf/www/ZhuWSJ021,DBLP:conf/aaai/LiHW18}.

Along this line of thought, \textit{if a GNN regularization negates $\mathcal{L}_{GCN}$, it incorporates a high-pass filter}, and we propose:
\begin{theo}
\label{theorem}
Suppose a complement graph $\bar{\mathcal{G}}=(\mathcal{V},\bar{\mathcal{E}})$.
$A_s$ is the adjacency constructed by aribitrary $\mathcal{E}_s \subset \bar{\mathcal{E}}$.
The Laplacian regularization built on $A_s$, i.e., $\mathcal{L}_{s} = \min_{Z} \tr \left( Z^T \tilde{\mathrm{L}}(A_s) Z \right)$ produces a high-pass filter effect on the original graph $\mathcal{G}=(\mathcal{V}, \mathcal{E})$'s signal $Z$.
\end{theo}
\begin{proof}
Suppose a edge set $\mathcal{E}^*$ composed by the arbitrary subset of the complement edges and the original edge set, that $\mathcal{E}^* = \mathcal{E}_s \cup \mathcal{E}$. Then we can rewrite $A_s = - A + A^* $, where $A^*$ is the adjacency matrix of $\mathcal{E}^*$ and $A_s$ is for $\mathcal{E}_s$. The objective of the Laplacian regularization with respect to $A_s$ is formulated as follows:
\begin{align}
    \mathcal{L}_{s} &= \min_Z  \tr\left(Z^T\tilde{\mathrm{L}}\left(-A + A^*\right)Z\right) \\
    &= \min_Z -\tr\left(Z^T\tilde{\mathrm{L}}(A)Z\right) + \tr\left(Z^T\tilde{\mathrm{L}}\left(A^*\right)Z\right).
    \label{equ:prop}
\end{align}
Using Lemma~\ref{lemma2} and Lemma~\ref{lemma1}, $\tr(Z^T\tilde{\mathrm{L}}(A)Z)$ is a low-pass filter $\Phi_A(\lambda) = (1-\lambda)$. The former half of Eq.~(\ref{equ:prop}) negates this low-pass filter, and thus it constructs a high-pass filter. 
\end{proof}
{


Any given subset of complementary edges equals the result of 
subtracting the original edge set from the union of the arbitrary subset and original edge set, making the Laplacian regularization built on it consist of high-pass filters from the original graph and low-pass filters from the corresponding complementary graph.
Note that this theorem demands no restrictions on subsets, making its realization straightforward on any graphs.
\eat{
The intuition behind this theorem is simple, regularization on an arbitrary subset of complementary edges can be thought of as subtracting the original set of edges from the union of the arbitrary subset and the original set of edges.
Given this, Laplacian regularization built on arbitrary subsets high-pass filters the original spectrum and low-pass filters the spectrum of the union set consisting of the original edge set and the subset.}

We want to emphasize that Eq.~(\ref{equ:prop}) may fall into an unexpected calculation as  $\tr\left(Z^T\tilde{\mathrm{L}}\left(A^*\right)Z\right)$ when fixing the subset during the training. Instead, one can apply random sampling at each layer to avoid this, therefore, we develop the following sampling methods on $\bar{\mathcal{E}}$.
}

\subsection{Sampling from the Complement Graph}
Constructing a high-pass filter from a complement graph is not trivial. 
According to Theorem~\ref{theorem}, it is unnecessary and impractical to include the whole complement graph. Instead, we only need the high-pass character from any subset of the complement. 
Therefore, 
we propose two efficient sampling strategies for the complement graph from two perspectives, including a node level and an edge level.

\textit{Node-based sampling strategy} traverses the nodes set. 
Given a fixed sampling multiple $S$, for each node $v_i$ in $\mathcal{V}$, we sample $S$ nodes from $\mathcal{V}$, if and only if they are not connected in the graph. Then, we obtain a sampled complement with $|\mathcal{V}|$ nodes and maximum of $S\times|\mathcal{V}|$ edges.
 
\textit{Edge-based sampling strategy}
further considers the influence of node degrees. 
From a local perspective, suppose a node connects to more neighbors than other nodes. We argue it could be more influenced by the original low-pass filter and more edges from the complement graph should be helpful to make the high-pass filter works. 
This strategy iterates all the nodes in each of the edges, in which the node degrees embed. 
Given a fixed sampling multiple $S$, for each node $(u,v)$ in each edge, we sample $S$ nodes from $\mathcal{V}$, if and only if either $u$ and $v$ does not connected to them. By this means, we sample the complement with $|\mathcal{V}|$ nodes (since the graph is connected) and $2S\times|\mathcal{E}|$ edges.

In the following content, we signify the sampled graph obtained from the both methods as $\mathcal{G}_{s}=(\mathcal{V}_{s},\mathcal{E}_{s})$.
Please refer to the following Algorithm~\ref{alg:sample} for the sampling details.

\begin{algorithm}[h]
\caption{Sampling Strategy $\mathtt{SAMPLE}$}
\label{alg:sample}
\textbf{Input}: the raw graph $\mathcal{G}=(\mathcal{V},\mathcal{E})$ and sampling hyper-parameter $S$. 
\begin{algorithmic}[1]
\STATE Initialize $\mathcal{E}_{s} \leftarrow \emptyset$.
\IF{sampling based on nodes}
\FOR{each node $u$ in $\mathcal{V}$}
\STATE sample $S$ edges $\mathcal{E}_{u}=\{(u,v_i)\}$ from $\bar{\mathcal{E}}$
\STATE $\mathcal{E}_{s} \leftarrow \mathcal{E}_{s} \cup \mathcal{E}_{u}$
\ENDFOR
\ENDIF
\IF{sampling based on edges}
\FOR{each edge $(u,v)$ in $\mathcal{E}$} 
\STATE sample $S$ edges $\mathcal{E}_{u}=\{(u,v_i)\}$ from $\bar{\mathcal{E}}$
\STATE sample $S$ edges $\mathcal{E}_{v}=\{(u_i,v)\}$ from $\bar{\mathcal{E}}$
\STATE $\mathcal{E}_{s} \leftarrow \mathcal{E}_{s} \cup \mathcal{E}_{u} \cup \mathcal{E}_{v}$
\ENDFOR
\ENDIF
\RETURN Sampled graph $\mathcal{G}_{s}=(\mathcal{V},\mathcal{E}_{s})$
\end{algorithmic}
\end{algorithm}


\subsection{Constructing Regularization on the Sampled Complement}
\label{sec:const_clar}
Based on the sampled complement graph, we propose Complement LAplacian Regularization~(CLAR) as a plugin layer to enhance GNNs. CLAR is composed of two regularization components to cover both low-pass and high-pass functions. In formal, we use $\mathcal{L}_{CLAR}$ to represent the CLAR regularization, then,
\begin{equation}
    \mathcal{L}_{CLAR} = \beta \mathcal{L}_{com} + \alpha \mathcal{L}_{ori},
    \label{equ:clar}
\end{equation}
where $\mathcal{L}_{com}$ and $\mathcal{L}_{ori}$ are high-pass and low-pass regularization respectively, and $\alpha$ and $\beta$ are hyper-parameters.

The high-pass regularization $\mathcal{L}_{com}$ is a Laplacian regularization built on the extracted complement $\mathcal{G}_{s}$,
\begin{equation}
    \mathcal{L}_{com}=\tr \left(f(X,A)^T\mathrm{L}(\tilde{A_s})f(X,A)\right),
    \label{equ:com}
\end{equation}
where $\tilde{A_s}$ is the normalized adjacency matrix reduced from $\mathcal{G}_{s}$. 

{
In Theorem~\ref{theorem}, we show that  constructing $A_s$ with randomness can prevent $\mathcal{L}_{com}$ from converging to unexpected low-pass filters from $A^*=A+A_s$. 
Our proposed  sampling strategy $\mathtt{SAMPLE}$ is well aligned with the need for this theorem.
However, the frequency components of $\tilde{\mathrm{L}}(A^*)$, $\lambda^*$, is incalculable at each training iteration. 
We therefore introduce a low-pass filter on $\tilde{\mathrm{L}}(A)$ to amend this.}
Specifically, we append a modified Laplacian regularization, where we replace Eq.~(\ref{equ:reg}) with GNNs output $f(X,A)$, 
\begin{equation}
    \mathcal{L}_{ori}=\tr \left(f(X,A)^T \mathrm{L}(\tilde{A}) f(X,A)\right).
    \label{equ:ori}
\end{equation} 
{
Here, $\mathcal{L}_{ori}$ (Eq.~(\ref{equ:ori})) relies on hidden representations (i.e., $f(X,A)$), as opposed to operating on the fixed structure in traditional GNNs.
}

In general, incorporating GNNs with a high-pass filter provided by CLAR can be achieved by the objective function:
\begin{equation}
    \mathcal{L} = \mathcal{L}_{GNN} + \mathcal{L}_{CLAR}. 
\end{equation} 
$\mathcal{L}_{GNN}$ is the classification loss of GNNs, e.g.,
\begin{equation}
    \mathcal{L}_{GNN}=\mathtt{CrossEntropy}\left(f(X,A)_{train},Y_{train}\right).
    \label{equ:semi_loss}
\end{equation}
This formulation is in line with the semi-supervised node classification paradigm, i.e., Eq.~(\ref{equ:semi}).
Obviously, our method is completely a post plugin for GNNs, and can be applied to any GNN backbones. 
The overall optimization process of deploying CLAR is listed in Algorithm~\ref{alg:clar}.

\begin{algorithm}[h]
\caption{Optimization process of GNNs with CLAR}
\label{alg:clar}
\textbf{Input}: the raw graph $\mathcal{G}=(\mathcal{V},\mathcal{E})$ with adjacency matrix $A$ and node features $X$, GNN function $f_{\Theta}$ with trainable parameters $\Theta$, sampling strategy $\mathtt{SAMPLE}$, sampling hyper-parameter $S$, adjustment hyper-parameters $\alpha, \beta$, learning step $\eta$, and the maximum training epoch $E$.
\begin{algorithmic}[1]
\STATE Initialize the parameter of the GNN function: $\Theta \leftarrow \Theta_0$
\FOR{each epoch $e$ in $0,1,\cdots E$}
\STATE get the hidden representations: $H \leftarrow f_{\Theta_e}(X,A)$
\STATE compute GNN loss function $\mathcal{L}_{GNN}\leftarrow\mathtt{CrossEntropy}\left(H_{train},Y_{train}\right)$
\STATE get the sampled complement: $\mathcal{G}_{s} \leftarrow \mathtt{SAMPLE}(\mathcal{G},S)$
\STATE calculate two regularization terms: $\mathcal{L}_{com} \leftarrow \tr(H^T\tilde{A}_sH)$ and $\mathcal{L}_{ori}\leftarrow \tr(H^T\tilde{A}H)$ 
\STATE $\mathcal{L}_{CLAR} \leftarrow \beta \mathcal{L}_{com} + \alpha \mathcal{L}_{ori}$, $\mathcal{L}\leftarrow\mathcal{L}_{CLAR}+\mathcal{L}_{GNN}$
\STATE calculate the gradient $\frac{\partial \mathcal{L}}{\partial \Theta_{e}}$ 
\STATE update $\Theta_{e+1}\leftarrow\Theta_{e} + \eta \frac{\partial \mathcal{L}}{\partial \Theta_{e}}$
\IF{the criterion for stop training is satisfied}
\STATE get the optimized parameter $\Theta^* \leftarrow \Theta_{e+1}$
\ENDIF
\ENDFOR
\RETURN the optimized GNN function $f_{\Theta^*}$
\end{algorithmic}
\end{algorithm}

\textbf{Remark.}\quad
The hyper-parameter $\alpha$ and $\beta$ can shaping the filter. For example, increasing $\alpha$ enhances the low-frequency components behaving the same as Laplacian regularization (i.e., Eq.~(\ref{equ:reg})). On the contrary, increasing $\beta$ enhances the high-frequency components. 
Additionally, since the sampling multiple $S$ is not directly related to $\{\lambda^*\}$, it is not sensitive to CLAR, which is consistent with our experiments for Table~\ref{tab:spearman}.  
\input{tab/tab_data}

\subsection{Spectral Analysis for Other Regularization Methods}
We further provide spectral analysis of CLAR and other regularization methods to indicate that only our proposal obtains a high-pass filter. We omit the hyper-parameter ahead of the regularization methods, e.g. $\gamma$ in Eq.~(\ref{equ:semi}).

\textit{Network Lasso~(NL)}
,i.e. Eq.~(\ref{gl_optim})~\cite{nlasso}, is identical to graph Laplacian regularization~ Eq.~(\ref{equ:semi_loss})~\cite{kipf_deep_2020}, a low-pass filter covered in GCN (Eq.~(\ref{equ:gcn_optim})) and CLAR (Eq.~(\ref{equ:ori})). Therefore, its improvement stays trivial.
\begin{equation}
    \mathcal{L}_{NL}=\min_{Z} \tr(Z^T\tilde{\hat{L}}Z).
\label{gl_optim}
\end{equation}

\textit{P-reg~(P)}
, i.e. Eq.~(\ref{p_optim}), penalties the difference of $Z$ and GCN-smoothed representations $\tilde{\hat{A}}Z$, which is equivalent to squared Laplacian regularization using squared error~\cite{DBLP:conf/aaai/0002MC21}. 
\begin{equation}
    \mathcal{L}_{P}=\min_{Z} \tr(Z^T\tilde{\hat{L}}^T\tilde{\hat{L}}Z).
    \label{p_optim}
\end{equation}

The convolution matrix of P-reg is the adjacency matrix of $\tilde{\hat{L}}^2$. According to Lemma~\ref{lemma1}, its frequency response is $\Phi_{P-reg}(\lambda)\approx1-\lambda^2$, as the frequency component of $\tilde{\hat{L}}^2$ is $\lambda^2$. Therefore, P-reg is still a low-pass filter with more low-frequency components than NL.

\textit{MADReg~(MR)}
, i.e. Eq.~(\ref{mad_optim}), proposes a penalty on the difference between higher-order neighbors and lower-order neighbors to alleviate over-smoothing. Specifically, MADReg obtains the $k$-order neighbors by stacking $k$ layers (e.g., GCN layers). Practically, MADReg takes $8\geq$ and $\leq 3$ for higher-order and higher-order, respectively. 
\begin{equation}
    \mathcal{L}_{MR}=\min_{Z} \tr(Z^T((Q-\tilde{\mathrm{L}}(\tilde{A^7}))-\tilde{\mathrm{L}}(\tilde{A^3}))Z)
    \label{mad_optim}
\end{equation}
$Q$ is a matrix of ones, and $(Q-\tilde{\mathrm{L}}(\tilde{A^7}))$ signifies higher-order neighbors that $k\geq8$. We approximate the Laplacian of the power adjacency matrix to the power of the Laplacian matrix, $\tilde{\mathrm{L}}(\tilde{A^k})\approx\tilde{\mathrm{L}}(\tilde{A})^k$. Then we have 
\begin{align*}
    \mathcal{L}_{MR}&\approx\min_{Z} \tr(Z^TQZ)
    -\tr(Z^T\tilde{\mathrm{L}}(\tilde{A})^7Z)\\
    &-\tr(Z^T\tilde{\mathrm{L}}(\tilde{A})^3Z).
\end{align*}
Leaving the independent first term, we derive the frequency responses of MADReg, that $-\Phi_{MR}(\lambda)\approx(1-\lambda)^3+(1-\lambda)^7$, which is a negation of two more extreme low-pass filters, that makes the high-frequency components not distinctive.

\textit{AdaGraph~(AG)}
, i.e. Eq.~(\ref{ag_optim}), modifies Network Lasso (Eq.~(\ref{gl_optim})) via adding and dropping edges to $A$. \textit{DropEdge} is a special case of AG, when $A_{add}=[0]_{N\times N}$.
\begin{equation}
    \mathcal{R}_{AG}=\min_{Z} \tr(Z^T\tilde{\mathrm{L}}(A-A_{drop}+A_{add})Z)
    \label{ag_optim}
\end{equation}
Practically, $l1$-norm regularization is constrained on $A_{add}$ and $A_{drop}$, making them sparse~\cite{DBLP:conf/aaai/ChenLLLZS20}. Consequently, the low-pass filter of $\tilde{\mathrm{L}}(A)$ dominates the optimization problem.

In total, CLAR is unparalleled among other regularization methods for incorporating a high-pass filter.

\textbf{Remark.}\quad  
We want to emphasize that adding CLAR to GNNs does not equal adding a regularization on the complete graph.
Direct usage of the complete graph may fail to capture the desired high-pass information because it is often too dense.
Therefore, a sampling process is required.  
Towards this, we randomly sample complement edges in each training epoch, resulting in high-pass filtering according to
Theorem~\ref{theorem}.
Besides, such a random sampling can ensure the overall optimization does not depend on any particular edge sets.



%% file: tab/tab_data.tex
\begin{table*}[t]
    \centering
    \caption{Statistics of Data Sets with Homophilous and Heterophilous}
    \begin{tabular}{ccccccccccccc}
    \toprule
    &\multicolumn{7}{c}{Homophilic graphs}&\multicolumn{5}{c}{Heterophilic graphs}\\
    \cmidrule(r){2-8}   \cmidrule(l){9-13}
    & CiteSeer & Cora & Computers & CS & Photo & Physics & PubMed & Actor & Chameleon & Squirrel & Cornell & Texas \\
    \midrule
    $|\mathcal{V}|$ & 4,732 & 2,708 & 13,752 & 18,333 & 7,650 & 34,493 & 19,717 & 7,600 & 2,277 & 5,201 & 183 & 183 \\
    $|\mathcal{E}|$ & 13,703 & 5,278 & 491,722 & 81,894 & 238,162 & 247,962 & 44,338 & 238,162 & 72,202 & 434,146 & 298 & 325 \\
    $h$ & 0.736 & 0.810 & 0.777 & 0.8081 & 0.827 & 0.9314 & 0.802 & 0.219 & 0.233 & 0.224 & 0.305 & 0.108 \\
    $c$ & 6 & 7 & 10 & 15 & 8 & 5 & 3 & 5 & 6 & 5 & 5 & 5 \\
    \bottomrule 
    \end{tabular}
    \label{tab:graph_stat}
\end{table*}

%% file: tex/6_experiment.tex
\input{tab/tab_filter}
\section{Experiments}
\label{sec:exp}
In this section, we conduct experiments on homophilic and heterophilic graphs. 
\footnote{Our implementation is available at  \url{https://github.com/sajqavril/Complement-Laplacian-Regularization}.}
Table~\ref{tab:graph_stat} shows the statistics of the two types of graphs. The experiments aim to answer the following research questions:
\begin{itemize}
    \item[Q1.] Does CLAR better express high-frequency components?
    \item[Q2.] Does CLAR help to express graphs with heterophily?
    \item[Q3.] Is CLAR capable of alleviating over-smoothing?
    \item[Q4.] How does CLAR affect topological robustness?
\end{itemize}
 

\textit{Homophilic graphs.}
(1) Citation networks: Cora, CiteSeer, and PubMed~\cite{DBLP:conf/icml/YangCS16}, (2) Amazon networks: Computers and Photo~\cite{DBLP:journals/corr/abs-1811-05868}, networks of whether goods are frequently brought together. 
and (3) Coauthor networks: Physics and CS~\cite{DBLP:conf/aaai/ChenLLLZS20}.

\textit{Heterophily graphs.} 
(1) Actor, the actor-only induced subgraph in a film-director-actor-writer network~\cite{DBLP:conf/iclr/PeiWCLY20}. (2) Cornell and Texas from WebKB~\cite{DBLP:conf/iclr/PeiWCLY20}, networks of web pages, and the hyperlinks between them. And (3) Squirrel and Chameleon from WikipediaNetwork~\cite{DBLP:journals/compnet/RozemberczkiAS21}, including Wikipedia pages and their hyperlinks. 

\textbf{Remark for computational complexity.}\quad
For GNN models, the training time complexity of GCN, GPRGNN, and Cheby-GCN is linearly related to the propagation step due to their iterative formulations. However, BernNet is quadratic. $\mathcal{O}(|\mathcal{E}||\mathcal{V}|)$ is expected in each propagation layer if utilizing sparse matrix multiplication, $\mathcal{O}(|\mathcal{V}|^3)$ otherwise. CLAR takes time complexity $\mathcal{O}(d|\mathcal{V}|)$ and no space complexity, where $d$ is the hidden dimension. CLAR and other regularization methods stay the same time and space complexity to the backbone GNNs during inference. Therefore, we include the regularization methods and Cheby-GCN and GPRGNN with two layers in our experiments. For a comprehensive examination of CLAR, we further consider GAT and SAGE.

\subsection{Fitting on Artificial Graph Signals}
\label{exp:art}
For Q1, we artificialize low-pass, high-pass, band-pass, and band-reject filters on Chameleon and Squirrel and examine the fitness of different approaches.
In detail, we eigendecomposite the Laplacian matrix and apply the filters, i.e., $h(\cdot)$, on the graphs signals that $U diag(h(\lambda))U^T$. Afterward, we take the original graph signals as input and converge to the filtered signals using the mean squared error (MSE).
Experiments adopt the early-stop mechanism with a maximum of $1000$ epochs and patience of $50$ epochs and optimize using Adam with the fixed learning rate of $0.1$. The compared approaches are set as follows.

\textit{Baselines (vanilla).} We adopt two-layer GCN, GAT, and SAGE as the baselines. Attention heads for GAT in each layer are 8 and 1. We accept all the neighbors in SAGE to avoid the influence of sampling. Since the parameters in SAGE are doubled for the separation of central nodes and the aggregated neighbors, SAGE is an enhanced GCN. We append CLAR on the final output of the baselines. The hidden dimension is 32 for all. 

\textit{Ours (+~CLAR).} CLAR applied on the same eigendecomposition with the baselines for fair comparisons. We figure out an empirical solution since $\alpha$ and $\beta$ are too extensive to tune. Particularly, we clamp the loss value of CLAR within $[0.0, 1.0]$ for numerical stability and find that $\alpha,\beta \in [0,1,2]$ is enough to search for CLAR to assist GNNs. 

Table~\ref{tab:filter} shows the remained MSE after training, which is the lower, the better. All the baselines with CLAR outperform the vanilla versions, as CLAR expresses more frequency components. Among the four filters, the high-pass and the band-reject compose more high-frequency components than others, where the baselines can benefit more from CLAR. Besides, we observe that GAT lacks frequency explanation the most and also improves the most significant for supplying frequency components. SAGE is slightly less satisfied than GCN, due to the doubled parameters and separation.

\subsection{Expressing Real-World Graph Signals}
\label{sec:exp_real}
To answer Q2, we conduct semi-supervised node classification for real-world homophilic and heterophilic graphs. 
We randomly split all the data sets into a 60\% training set, 20\% validation set, 20\% test set, and report the average accuracy~(micro F1 score) after $50$ runs. 
Then, we introduce the implementation of the baselines and CLAR (+~CLAR), and present other compared methods in the following.

\input{tab/tab_node_classification}

\textit{Spectral filters.} Cheby-GCN (Cheby)~\footnote{https://github.com/pyg-team/pytorch\_geometric.git} and GPRGNN (GPR)~\footnote{https://github.com/jianhao2016/GPRGNN/} are set to $2$-order for equal computation.

\textit{Regularization methods.} 
We consider all the regularization methods in Sec.\ref{sec:relate} and append them on GCN.
We implement DropEdge~(DE)~\footnote{https://github.com/pyg-team/pytorch\_geometric.git} with a $0.5$ dropout rate.
P-reg~(P)~\footnote{https://github.com/yang-han/P-reg/} is set to $\lambda=0.1$ for citation networks and $0.5$ for the rest. 
Network Lasso~(NL) is identical to CLAR when $\alpha=1, \beta=0$. 
We realize MADReg~(MR) according to the original paper due to the absence of the source code and apply a clamp on the regularization values for numerical stability, same to CLAR. 
Note that we report our experimental results of  AdaGraph~(AG) using the setting from the original papers
due to the lacking of implementation.
Following this,  Planetoid splits are used for citation networks, where the sizes of  training/validation/test are 20/50/1000, respectively.
In the co-author networks, 20 nodes are randomly sampled for training, 30 nodes are for validation and the rest are for testing in each class. 
We summarize all results in Table~\ref{tab:ag_acc}, where CLAR outperforms AG consistently.

\begin{table}[h]
\caption{Comparisons to AG on Node Classification (Accuracy \%)}
\centering
\begin{tabular}{ccccccc}
\toprule
& CiteSeer & Cora & CS & Physics & PubMed\\
\midrule
GCN & 70.3 & 81.5 & 89.8 & 92.8 & 76.3 \\
w/ AG & 69.7 & 82.3 & 90.3 & 93.0 & 77.4\\
w/ CLAR & \textbf{71.5} & \textbf{82.6} & \textbf{91.4} & \textbf{93.9} & \textbf{78.5}\\

\midrule
GAT & 72.5 & 83.0 & 85.5 & 91.1 & 75.9 \\
w/ AG & 69.1 & 77.9 & 86.6 & 91.4 & 76.6\\
w/ CLAR & \textbf{72.7} & \textbf{83.3} & \textbf{87.3} & \textbf{92.0} & \textbf{77.1}\\

\midrule
SAGE & 67.4 & 78.9 & 90.1 & 93.0 & 75.2 \\
w/ AG & \textbf{69.4} & 80.2 & 90.3 & 92.7 & 77.2\\
w/ CLAR & 68.7 & \textbf{80.2} & \textbf{91.7} & \textbf{93.7} & \textbf{77.7}\\

\bottomrule
\end{tabular}
\label{tab:ag_acc}
\end{table}

In Table.~\ref{node_classification}, all the baselines outperform the original versions by CLAR. Significantly heterophilic graphs are improved by more high-frequency components.
Since GNN baselines satisfy the low-frequency components, the homophilic graphs improve less. $2$-order Cheby-GCN and GPRGNN are not enough to express the filters, except for the low-frequency components (homophilic graphs), particularly on Computers and Photo. Even though all the other regularization methods enhance GCN to some extent, their increments are marginal compared to CLAR, because only CLAR contains a high-pass filter.
The tuned hyper-parameters 
indicate that a greater $\alpha$, i.e. the low-pass, is beneficial for homophilic graphs, and a larger $\beta$, i.e. the high-pass helps with heterophily.

\textit{Sensitivity of the sampling hyper-parameter $S$.}
Recall that we assume the classification performance is robust to $S$ in Sec.~\ref{sec:const_clar}.
We further implement an experiment to verify this using Spearman correlation analysis, as shown in Table~\ref{tab:spearman}.
Here, we set the range of $S$ as $\{1,2,4,8,16,32\}$. It is obvious that the relation between the performance and  $S$ is neglectable, demonstrating the robustness of $S$.

\label{sec:reprod}

\begin{table}[h]
    \centering
    \caption{Spearman Correlation Coefficients of $S$ and Accuracy (\%)}
    \begin{tabular}{cccccc}
    \toprule
    & CiteSeer & Cora & CS & Physics & PubMed \\
    \midrule
    GCN & -0.078 & -0.059 & -0.034 & 0.166 & 0.182\\
    GAT & -0.084 & 0.200 & -0.104 & -0.042 & -0.063\\
    SAGE & 0.009 & -0.002 & 0.081 & 0.023 & -0.004\\
    \bottomrule
    \end{tabular}
    \label{tab:spearman}
\end{table}

\subsection{Tackling Over-smoothing} 
For Q3, we apply CLAR on GCN with increasing stacking layers and test on a more challenging Planetoid split, where only $20$ fixed samples of each category are for training, while $500$ and $1000$ samples are for validation and test~\cite{DBLP:conf/icml/YangCS16}. 
Over-smoothing might be more severe due to lacking supervised signals under this context. 
As introduced earlier, GNNs quickly become more low-pass with increasing layers.
To tackle this condition, we introduce more high-frequency components through larger loss values that we loose the clamp to $[0.0, 10.0]$ and restrain $\alpha,\beta \in \{0.0, 0.1, 0.2, 0.5, 0.8, 1.0\}$ to retain numerical stability. 
CLAR~(10) denotes this modified CLAR, and CLAR~(1) introduced earlier.
Besides, we include other regularization methods claiming against over-smoothing, DropEdge with a dropout ratio of $0.2$, and MADReg is set to be the same as $\alpha$ in CLAR~(10). 

We report the test accuracy of GCN with the regularization methods on Cora along with increasing layers in Fig.~\ref{fig:over_smoothing}.
According to the results, CLAR helps GCN beyond the original version and out-perform DropEdge around the early and more deep layers. MADReg performs slightly weaker. Importantly, CLAR is superior to MLP even in the deep layers, certifying the adjustable filters within its effect.

\begin{figure}[h]
    \centering
    \includegraphics[trim=0.3cm 0.4cm 0.3cm 0.3cm,clip,width=0.48\textwidth]{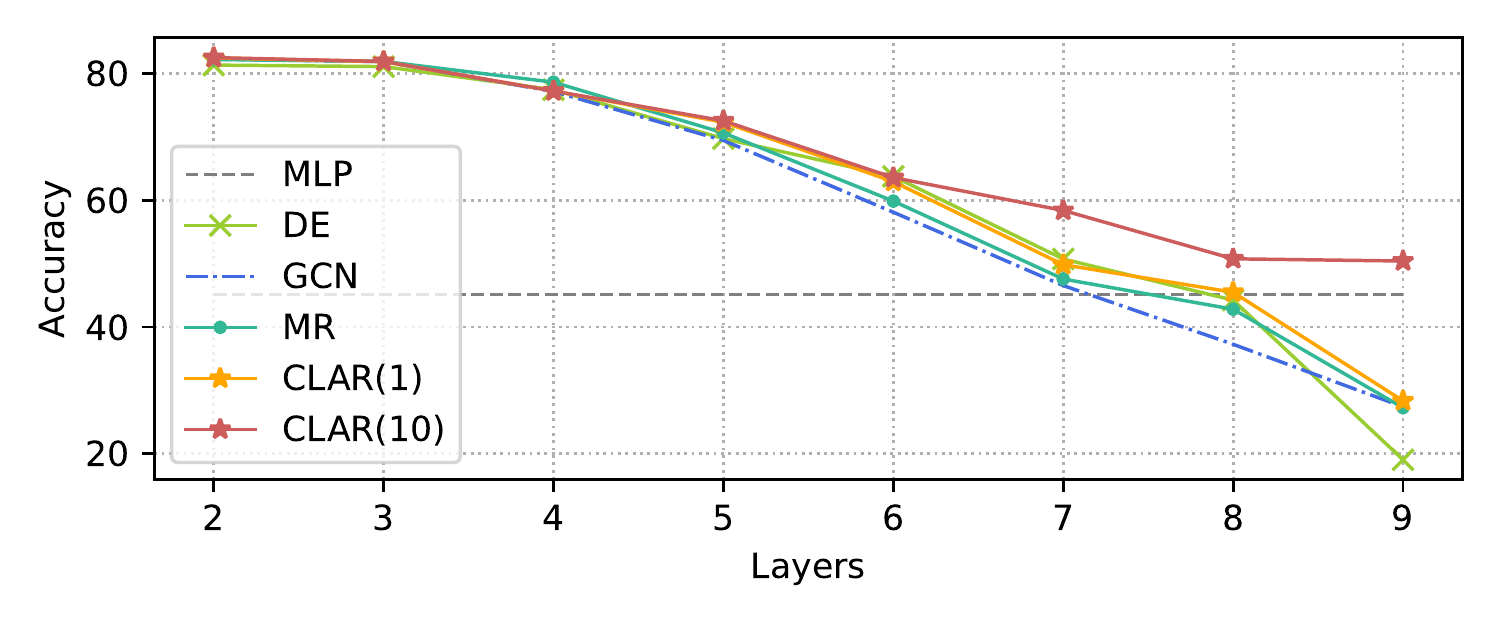}
    \caption{Test accuracy (\%) of GCN with various regularizers. Remarkably, only CLAR is always beyond MLP.}
    \label{fig:over_smoothing}
\end{figure} 

\subsection{Boosting Topological Robustness}
For Q4, we evaluate CLAR against topological noise, i.e., edges missing. This situation is common; in practice, only partial connections are observed often. Due to the random sampling strategies, CLAR is naturally steady to this. We successively drop the $10\%$ ratio of the edges to imitate the situation and evaluate CLAR on two-layer GCN. The results are summarized in Figure~\ref{fig:drop} with using Planetoid splits.

Along with increasing dropout ratios, CLAR helps to maintain performance. Particularly after dropping more edges, CLAR and CLAR* benefit significantly. 
Finally, we provide a spectral explanation for this empirical observation; when the low-frequency components are getting sparse, e.g., dropping edges, utilizing the high-pass filters might be beneficial to express the spectral components.  


\begin{figure}[h]
    \centering
    \includegraphics[trim=0.3cm 0.35cm 0.3cm 0.35cm,clip,width=0.48\textwidth]{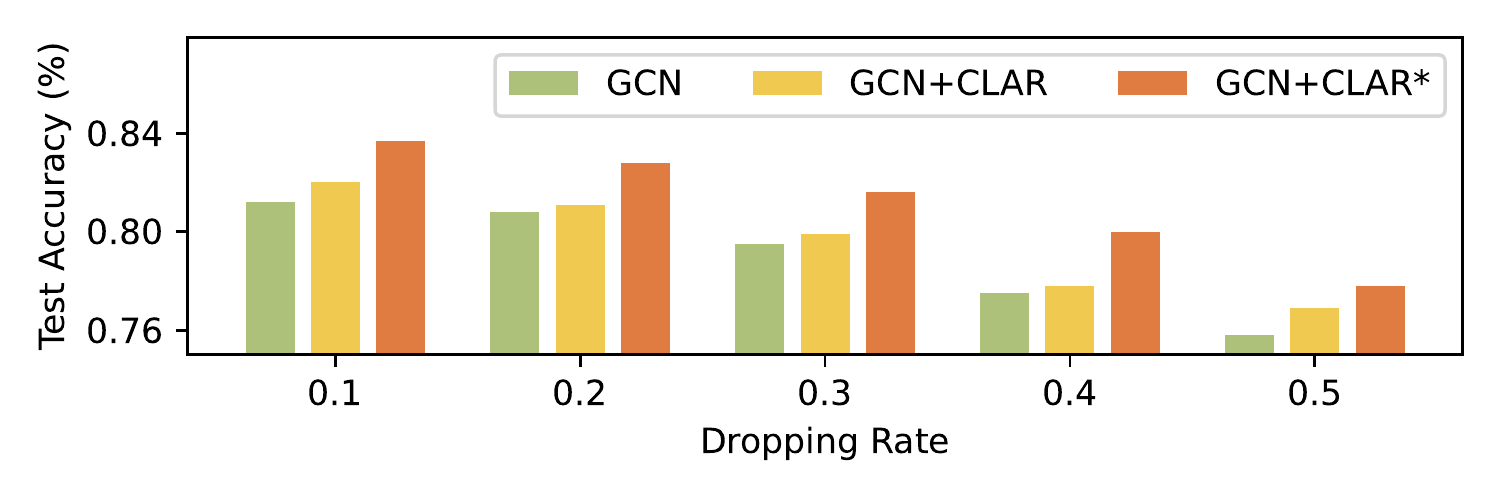}
    \caption{Test accuracy (\%) of GCN+CLAR with increasing dropping edges rate. * denotes magnifying CLAR for $10$ times.}
    \label{fig:drop}
\end{figure}



\subsection{ Sampling Strategies' Evaluation}
We proposed two types of sampling strategies from two perspectives,
namely node-based sampling and edge-based sampling. 
We argue that these two have similar behaviors w.r.t the overall performance.
To verify this, we deploy CLAR on a two-layer GCN on two homophilic datasets 
and two heterophilic datasets. 
We evaluate these two strategies using the following settings for the hyper-parameters:  $\alpha, \beta \in \{0, 0.01, 0.02, 0.05, 0.1, 0.2, 0.5, 1, 2\}$ and $S \in \{1, 2\}$, 
We use Pearson product-moment correlation coefficient to measure the difference between their performance on the  test set, as shown in Table~\ref{tab:hyper_sample}. 
Note that the value of  Pearson product-moment correlation coefficient
is within the range of $[-1,1]$, whose signs signify the positive or negative relations.
Normally, it is said to be a strong positive correlation when this coefficient is greater than $0.5$
~\cite{schober2018correlation}.
In Table~\ref{tab:hyper_sample}, all coefficients are greater than $0.8$
, indicating the difference between these two sampling strategies is neglectable.
One possible explanation is that our target graphs are insensitive to
these two sampling methods due to their sparseness.
In practice, we suggest using the edge-based sampling strategy when graphs are dense and the node-based one on sparse graphs.

\begin{table}[h]
    \centering
    \caption{Pearson product-moment correlation coefficients of the two sampling strategies}
    \begin{tabular}{c|cccc}
    \toprule
        Dataset &  Cora & CiteSeer & Chameleon & Squirrel \\
    \midrule
        Coefficients & 0.8067 & 0.9903 & 0.9382 & 0.8431 \\
    \bottomrule 
    \end{tabular}
    \label{tab:hyper_sample}
\end{table}

\subsection{Reproductivity}
We now report some of the details of implementation for reproductivity, which are fully available in our code repository.

Firstly, the sampling hyper-parameter $S$ is restricted within $\{1,2\}$, whose insensitivity is studied in Sec.~\ref{sec:exp_real}.
$\alpha$ and $\beta$ are used for adjusting the proportion of the high- and low- frequency components in CLAR. 
In random splits, they are within $\{0.0,1.0,2.0\}$, discussed in Sec.~\ref{exp:art}.
In Planetoid splits, they are $\{0.0,\pm 0.001,\pm 0.005\}$ to avoid over-fitting on the regularization because of the stronger sparsity, where the limited negative values can expand the hyper-parameter space.




%% file: tab/tab_filter.tex
\begin{table*}[t]
    \centering
    \caption{Mean squared error (MSE) of fitting artificial filters. With CLAR, GCN fits better on high-pass and band-rejection filters, which consist more high-frequency components than the other two.}
    \begin{tabular}{cccccccccc}
    \toprule
    & & \multicolumn{2}{c}{\textbf{High-pass}} & \multicolumn{2}{c}{Low-pass} & \multicolumn{2}{c}{Band-pass} & \multicolumn{2}{c}{\textbf{Band-rejection}} \\
    \cmidrule{3-10}
    & & \multicolumn{2}{c}{$1-\exp(-10\lambda^2)$} & \multicolumn{2}{c}{$\exp(-10\lambda^2)$} & \multicolumn{2}{c}{$\exp(-10(\lambda-1)^2)$} & \multicolumn{2}{c}{$1-\exp(-10(\lambda-1)^2)$} \\
    \cmidrule{3-10}
    & & Vanilla & \textbf{+~CLAR} & Vanilla & \textbf{+~CLAR} & Vanilla & \textbf{+~CLAR} & Vanilla & \textbf{+~CLAR} \\ 
    \midrule
    \multirow{3}{*}{Chameleon} & GCN  & 41.6940 & 41.5680 & 3.7930 & 3.6899 & 29.3921 & 29.3701 & 16.9194 & 16.8692 \\
                               & GAT  & 48.1473 & 44.3101 & 3.9987 & 3.9906 & 32.3080 & 31.7117 & 24.9859 & 18.1414 \\
                               & SAGE & 31.5833 & 31.5101 & 3.2877 & 3.2766 & 23.0391 & 22.7841 & 19.7868 & 19.6637 \\
    \midrule
    \multirow{3}{*}{Squirrel}  & GCN  & 38.1492 & 38.0042 & 3.6280 & 3.6042 & 29.5017 & 29.4274 & 13.6287 & 13.6055 \\
                               & GAT  & 39.9876 & 39.2110 & 4.8531 & 4.0223 & 30.1945 & 30.1945 & 18.6755 & 14.3854 \\
                               & SAGE & 33.2764 & 33.2138 & 2.3511 & 2.3502 & 26.2713 & 26.1866 & 11.1463 & 11.1266  \\ 
                              
    \bottomrule
    \end{tabular}
    \label{tab:filter}
\end{table*} 

%% file: tab/tab_node_classification.tex
\begin{table*}[t]
    \centering
    \caption{Overall performance comparisons on node classification (Accuracy \%)}
    \begin{tabular}{ccccccccccccc}
    \toprule
    & \multicolumn{3}{c}{Baselines} & \multicolumn{2}{c}{Spectral filters} & \multicolumn{4}{c}{Regularization methods} & \multicolumn{3}{c}{{\textbf{+~CLAR}}} \\
    \cmidrule(l){2-4} \cmidrule(l){5-6} \cmidrule(l){7-10} \cmidrule(l){11-13}
         & GCN & GAT & SAGE & Cheby & GPR & NL & P & MR & DE & \textbf{GCN} & \textbf{GAT} & \textbf{SAGE} \\
    \midrule
    CiteSeer   & 79.15 & 80.51 & 79.02 & 79.70 & 74.39 & 79.62 & 79.83 & 79.89 & 78.29 & \textbf{80.65} & 81.57          & \textbf{80.65} \\
    Cora       & 87.24 & 87.87 & 87.34 & 83.44 & 82.46 & 87.61 & 87.83 & 87.70 & 86.06 & 88.23          & \textbf{88.60} & 88.56          \\
    Computers  & 83.22 & 82.80 & 87.38 & 86.96 & 89.32 & 83.46 & 83.23 & 83.19 & 83.50 & 83.81          & 83.75          & \textbf{88.21} \\
    Photo      & 88.39 & 91.34 & 93.95 & 92.07 & 94.05 & 88.86 & 89.12 & 89.72 & 90.03 & 90.03          & 92.21          & \textbf{94.16} \\
    PubMed     & 87.07 & 87.14 & 88.09 & 85.65 & 81.79 & 87.18 & 87.57 & 86.41 & 86.64 & 87.69          & 87.74          & \textbf{88.53} \\
    \midrule
    Actor      & 32.96 & 33.53 & 37.78 & 31.89 & 36.41 & 32.98 & 31.54 & 30.83 & 31.99 & 33.80          & 35.62          & \textbf{39.15} \\
    Chameleon  & 61.25 & 35.35 & 60.61 & 56.06 & 62.46 & 61.98 & 61.58 & 61.98 & 57.11 & \textbf{64.23}          & 42.57          & 63.15 \\
    Squirrel   & 45.73 & 24.13 & 50.37 & 38.13 & 48.71 & 46.88 & 47.64 & 47.22 & 39.48 & 48.60          & 28.05          & \textbf{52.05} \\
    Cornell    & 59.70 & 71.38 & 86.15 & 88.61 & 73.38 & 61.54 & 66.15 & 60.61 & 66.15 & 64.00          & 75.70          & \textbf{90.15} \\
    Texas      & 64.92 & 74.46 & 88.31 & 88.31 & 74.15 & 66.23 & 67.08 & 65.85 & 62.77 & 71.38          & 79.38          & \textbf{92.00} \\
    \bottomrule
    \end{tabular}
    \label{node_classification}
\end{table*}

%% file: tex/7_conclusion.tex
\section{Conclusion and Future work.}
Modern GNNs normally behave as low-pass filters and neglect the high-frequency components. In the context of regularization methods for GNNs, this paper proposes CLAR to enhance the high-frequency components efficiently. Based on our main theorem that the complement of the original graph incorporates a high-pass filter, we explore the complement and obtain a complementary Laplacian regularization term. Experiments verify that our proposal can help GNNs deal with over-smoothing and better express heterophilic graphs. 
We extensively discuss the spectral interpretability for existing regularization methods and point out that CLAR only possesses a high-pass filter within its effect. 

Next, the complement is worthy of exploring to model the graph spectrum better when low-frequency components are sparse, but high-frequency components might help. 
A high-frequency enhanced GNN architecture is also promising via the complement without increased parameters.